\documentclass{article}
\usepackage{microtype}
\usepackage{graphicx}
\usepackage{subfigure}
\usepackage{amsfonts}
\usepackage{algpseudocode}
\usepackage{hyperref}
\usepackage{placeins}
\usepackage{subcaption}
\newcommand{\diag}{\mathrm{diag}}

\usepackage{booktabs} 

\usepackage{hyperref}

\usepackage[accepted]{icml2025}

\usepackage{amsmath}
\usepackage{amssymb}
\usepackage{mathtools}
\usepackage{amsthm}

\usepackage[capitalize,noabbrev]{cleveref}

\theoremstyle{plain}
\newtheorem{theorem}{Theorem}[section]
\newtheorem{proposition}[theorem]{Proposition}
\newtheorem{lemma}[theorem]{Lemma}
\newtheorem{corollary}[theorem]{Corollary}
\theoremstyle{definition}
\newtheorem{definition}[theorem]{Definition}

\theoremstyle{remark}
\newtheorem{remark}[theorem]{Remark}

\usepackage[textsize=tiny]{todonotes}

\icmltitlerunning{From SGD to Spectra: A Theory of Neural Network Weight Dynamics}

\begin{document}

\onecolumn
\icmltitle{From SGD to Spectra: A Theory of Neural Network Weight Dynamics}




\begin{icmlauthorlist}
\icmlauthor{Brian Richard Olsen}{yyy}
\icmlauthor{Sam Fatehmanesh}{yyy}
\icmlauthor{Frank Xiao$^*$}{yyy}
\icmlauthor{Adarsh Kumarappan$^*$}{yyy}
\icmlauthor{Anirudh Gajula}{yyy}
\end{icmlauthorlist}

\icmlaffiliation{yyy}{California Institute of Technology, Pasadena California, USA}

\icmlcorrespondingauthor{Brian Richard Olsen}{bolsen@caltech.edu}
\icmlcorrespondingauthor{Sam Fatehmanesh}{sfatehma@caltech.edu}

\icmlkeywords{Machine Learning, ICML}

\vskip 0.3in


\printAffiliationsAndNotice{\icmlEqualContribution}

\begin{abstract}

Deep neural networks have revolutionized machine learning, yet their training dynamics remain theoretically unclear—we develop a continuous-time, matrix-valued stochastic differential equation (SDE) framework that rigorously connects the microscopic dynamics of SGD to the macroscopic evolution of singular-value spectra in weight matrices. We derive exact SDEs showing that squared singular values follow Dyson Brownian motion with eigenvalue repulsion, and characterize stationary distributions as gamma-type densities with power-law tails, providing the first theoretical explanation for the empirically observed `bulk+tail' spectral structure in trained networks. Through controlled experiments on transformer and MLP architectures, we validate our theoretical predictions and demonstrate quantitative agreement between SDE-based forecasts and observed spectral evolution, providing a rigorous foundation for understanding why deep learning works.

\end{abstract}

\section{Introduction}
\label{introduction}

Deep neural networks have fundamentally transformed machine learning, achieving unprecedented performance across diverse domains \citep{krizhevskyImageNetClassificationDeep2012, vaswaniAttentionAllYou2023a, jumperHighlyAccurateProtein2021}. Yet despite their empirical success, our theoretical understanding of how neural networks learn remains remarkably incomplete \citep{zhangUnderstandingDeepLearning2017, neyshaburExploringGeneralizationDeep2017}. Central to this understanding is the evolution of weight matrices, whose spectral properties—the distribution and dynamics of singular values—provide deep insights into optimization dynamics, generalization behavior, and implicit regularization \citep{penningtonResurrectingSigmoidDeep2017, martinImplicitSelfRegularizationDeep2018}.

At initialization, weight matrices exhibit well-characterized random matrix statistics described by the Marchenko-Pastur law, which characterizes the eigenvalue distribution of large random matrices, and related results from random matrix theory (RMT). However, training fundamentally alters these spectral properties, producing empirically observed `bulk+tail' structured distributions that correlate strongly with generalization performance \citep{Papyan2020,MartinMahoney2019}. Existing theoretical frameworks fail to explain this transformation: while RMT describes initial conditions and stochastic differential equations (SDEs) can model SGD, current analyses focus on scalar parameters or low-rank models, failing to capture the full matrix-valued dynamics \citep{mandtStochasticGradientDescent2018}. Most critically, no unified framework connects the microscopic stochastic dynamics of SGD to the macroscopic spectral evolution observed empirically.

In this paper, we bridge this gap by developing a continuous-time, matrix-valued SDE framework with carefully designed small-scale experiments that captures the full dynamics of singular value evolution under SGD. Our key contributions are:

\begin{enumerate}
    \item We derive exact SDEs for individual singular values under isotropic SGD noise, connecting the microscopic parameter updates to macroscopic spectral dynamics. Under the assumption of negligible gradients, we show that squared singular values follow a Dyson Brownian motion with $\beta = 1$, explaining the eigenvalue repulsion and spectral spreading.
    \item We characterize the stationary spectral distribution in the non-negligible gradient regime using mean-field theory. We prove that the limiting distribution follows a gamma-type density with power-law tails, recovering the empirically observed `bulk+tail' structure in trained networks and providing the first theoretical explanation for this empirical phenomenon.
    \item Through experiments on transformer \citep{vaswaniAttentionAllYou2023a}, vision transformer \citep{dosovitskiy2020image}, and MLP \citep{rumelhart1986learning} architectures, we demonstrate quantitative agreement between our SDE-based predictions and spectral evolution and propose an algorithm to forecast singular value dynamics from minimal gradient information.
\end{enumerate}

We show that SGD's stochastic noise acts like a ``spectral sculptor"—initially spreading eigenvalues apart via repulsion, then concentrating them into beneficial empirically observed `bulk+tail' structured patterns that enable generalization, connecting microscopic mini-batch randomness to macroscopic spectral evolution. Our work demonstrates how small-scale experimentation can unlock fundamental insights with implications for initialization strategies, optimization algorithm design, and understanding why deep learning works.

\section{Related Works}
\label{related_works}
\paragraph{Spectral Analysis of Neural Network Weights.}
RMT establishes that at initialization, weight matrices follow Wigner's semicircle law \citep{Wigner1955} and the Marčenko–Pastur distribution \citep{MarchenkoPastur1967}, with edge statistics governed by Tracy-Widom distributions \citep{tracyOrthogonalSymplecticMatrix1996}. Training induces pronounced deviations: singular-value spectra become highly anisotropic \citep{Saxe2013}, evolving `bulk+tail' structured distributions linked to class structure \citep{Papyan2020} and implicit regularization \citep{MartinMahoney2019}. Martin and Mahoney \citep{MartinMahoney2019} identified 5+1 phases of spectral evolution and showed that batch size affects spectral properties, with smaller batches leading to stronger implicit self-regularization. Extensions to empirically observed `bulk+tail' structured matrix ensembles \citep{auffingerPoissonConvergenceLargest2008} provide theoretical foundations past Gaussian universality classes.

\paragraph{SGD as Stochastic Dynamics.}
SGD can be approximated by stochastic differential equations (SDEs), with constant-rate SGD behaving like an Ornstein–Uhlenbeck process \citep{Mandt2017} and anisotropic noise structures enabling escape from sharp minima \citep{Zhu2019}. Weight updates have been mapped to Dyson Brownian motion \citep{Aarts2024}, explaining eigenvalue repulsion as a Coulomb-gas phenomenon. The mathematical foundation relies on Itô calculus for matrix functions \citep{g.w.stewartMatrixPerturbationTheory1990} and Fokker-Planck equations for interacting particle systems \citep{riskenFokkerPlanckEquationMethods1996}. The implicit regularization effects of gradient descent have been explored \citep{neyshaburSearchRealInductive2015}.

\section{Methodology}
\label{methodology}

Our approach transitions from discrete microscopic SGD dynamics to continuous macroscopic spectral evolution. We produce notation for the training of neural networks via a spatiotemporal interpretation of the evolution of the weight matrices with stochasticity: 
$ dW(x,t) = -\eta \frac{\partial \mathcal{L}}{\partial W(x,t)} dt + \sqrt{2\eta D_W(x,t)} \, d\mathcal{W}_W(x,t)$, 
where $\frac{\partial \mathcal{L}}{\partial W(x,t)} $ is the gradient of the loss, $\eta$ is the learning rate, $D$ is an effective diffusion constant described by  $d\mathcal{W}_W$, and $d\mathcal{W}_b$ are independent matrix/vector-valued Wiener processes, capturing the stochastic dynamics of SGD. We split our analysis into two cases: negligible gradient in loss and non-negligible gradient in loss. For the first limit (negligible gradient in loss, or $\frac{\partial \mathcal{L}(W)}{\partial W} \approx 0$), we perform SVD/eigenvalue decomposition and use Ito Calculus (see \hyperref[appendix:mainproofs]{Appendix 6.1}) to arrive at the result of \textbf{\hyperref[thm3.1]{Theorem 3.1}}:
\begin{theorem}[Stochastic Dynamics of Singular Values]
Let $W \in \mathbb{R}^{m \times n}$ evolve via stochastic gradient descent with noise. Then, the singular values \( \sigma_k(W) \) follow the SDE:
\begin{align*}
d\sigma_k(t) = \left[ -\eta u_k^T (\nabla_W \mathcal{L}) v_k + \eta D \left( \frac{m-n+1}{2\sigma_k} + \sum_{j \ne k} \frac{\sigma_k}{\sigma_k^2 - \sigma_j^2} \right)\right] dt + \sqrt{2 \eta D} d\beta_k(t) 
\end{align*}
where \( u_k, v_k \) are the singular vectors and \( D \) is the effective diffusion strength.
\end{theorem}
\label{thm3.1}

This theorem shows that under SGD noise, singular values behave like interacting particles that repel each other (the $\sum_{j \neq k}$ terms), explaining why the spectrum spreads out during training rather than collapsing. This SDE can then be mapped to Dyson-Brownian Motion processes (see \hyperref[appendix:dbm]{Appendix 6.6} whose statistics are described by the Marčenko–Pastur (\textbf{\hyperref[mplaw]{Lemma 6.9}}) and Tracy-Widom distributions respectively (\textbf{\hyperref[twlimit]{Lemma 6.11}}).

After deriving these microscopic underpinnings, we may consider the second case (non-negligible gradient in loss, or $\frac{\partial \mathcal{L}(W)}{\partial W} \neq 0$) and transition into a macroscopic limit by considering the empirical spectral density distribution $\rho(\lambda, t) $ such that: 
 $ \rho(\lambda, t) = \frac{1}{r} \sum_{k=1}^r \delta(\lambda - \lambda_k(t)) $
As $r \to \infty$, we assume that $\rho(\lambda, t)$ converges to a deterministic density function, normalized such that $\int \rho(\lambda, t) d\lambda = 1$.

To study the dynamics of the squared singular values $\{\lambda_j\}$ then, we adopt a mean-field perspective by assuming the effective influence of the complex loss function $\mathcal{L}(W)$ is captured by a potential $\mathcal{L}_{MF}$ that depends only on this set. This postulates the form: 
$ \mathcal{L}(W) \approx \mathcal{L}_{MF}(\{\lambda_j\}) = \frac{c}{2}  \sum_{j=1}^r (\lambda_j - \lambda^*)^2$, where these eigenvalues $\lambda_j$ are ``driven" to $\lambda^*$ as in \textbf{\hyperref[edgescaling]{Lemma 6.10}}.

Although deep neural network training objectives are generally nonconvex, it was shown in prior work \citep{NNConvexLandscape} that every stationary point of the original nonconvex problem coincides with the global optimum of a suitably defined subsampled convex program. As such, while the parameter-space landscape may admit many critical points, their spectral signatures at stationarity are governed by a convex variational principle. Consequently, modeling the large-\(r\) limit of the empirical spectral density \(\rho(\lambda,t)\) via a deterministic mean-field potential \(\mathcal L_{MF}\) is fully justified as in \textbf{\hyperref[rlimit]{Corollary 6.15}}.
Evaluating this spectral density distribution, neglecting its highest order terms, and we derive \textbf{\hyperref[thm3.2]{Theorem 3.2}}:
\begin{theorem}[Stationary Distribution of Singular Values]
\label{thm3.2}
Under the stationary mean-field approximation, the probability density function of the singular values follows a Gamma-type distribution:
\[
p_{\sigma}(\sigma) = 2 \frac{\left(\frac{\beta_1}{4\eta D}\right)^{\frac{m - n + 3}{4}}}{\Gamma\left(\frac{m - n + 3}{4}\right)} \sigma^{\frac{m - n + 1}{2}} e^{-\left(\frac{\beta_1}{4\eta D}\right) \sigma^2}
\]
where \( \beta_1 \) is an effective noise constant representing the mean-field restoring force of the gradient, $D$ is the diffusion constant, and the weight matrix $W$ is $m \times n$.
\end{theorem}

This gamma-type distribution with power-law tails captures the ``bulk+tail" structure observed empirically—most singular values cluster in a bulk region, while a few large values form the heavy tail that correlates with good generalization.

We present the proofs for both theorems in \hyperref[appendix:mainproofs]{Appendix 6.1}. Note that we treat \textbf{\hyperref[thm3.2]{Theorem 3.2}} phenomenologically by fitting to within bounds beyond the Tracy-Widom distribution predictions (see \hyperref[twlimit]{Corollary 6.1}). Thus, we take a ``mean-field" approach to deconstructing this system. See \textbf{\hyperref[lma1]{Lemma 6.17}} and \textbf{\hyperref[lma2]{Lemma 6.18}} for the estimation of the effective diffusion and noise constants governing the stochastic dynamics.

\section{Experiments}
\label{experiments}

\subsection{Experimental Setup}
\label{subsec:experimental_setup}
We validate our spectral‐SDE framework on three canonical architectures, namely a GPT-2 model; a Vision Transformer (ViT) \citep{dosovitskiy2020image}; and a MLP—all trained with SGD. This selection is motivated by extensive prior work showing `bulk+tail' spectra arise across MLPs, CNNs, and transformers \citep{MartinMahoney2019}, and that different architectures exhibit distinct spectral biases and mode‐learning rates \citep{Yao2022SpectralBias}. We initialize all weights from architecture‐specific Gaussian priors, then train GPT-2 on Shakespeare text \citep{shakespeare}, and ViT/MLP on MNIST \citep{lecun2010mnist} and CIFAR‐100 \citep{krizhevsky2009learning}. More experimental details are in \hyperref[appendix:models]{Appendix 8}.

\subsection{Singular Value Evolution Simulation and Analysis}
\begin{figure}[H]
    \centering
    \includegraphics[width=0.7\linewidth]{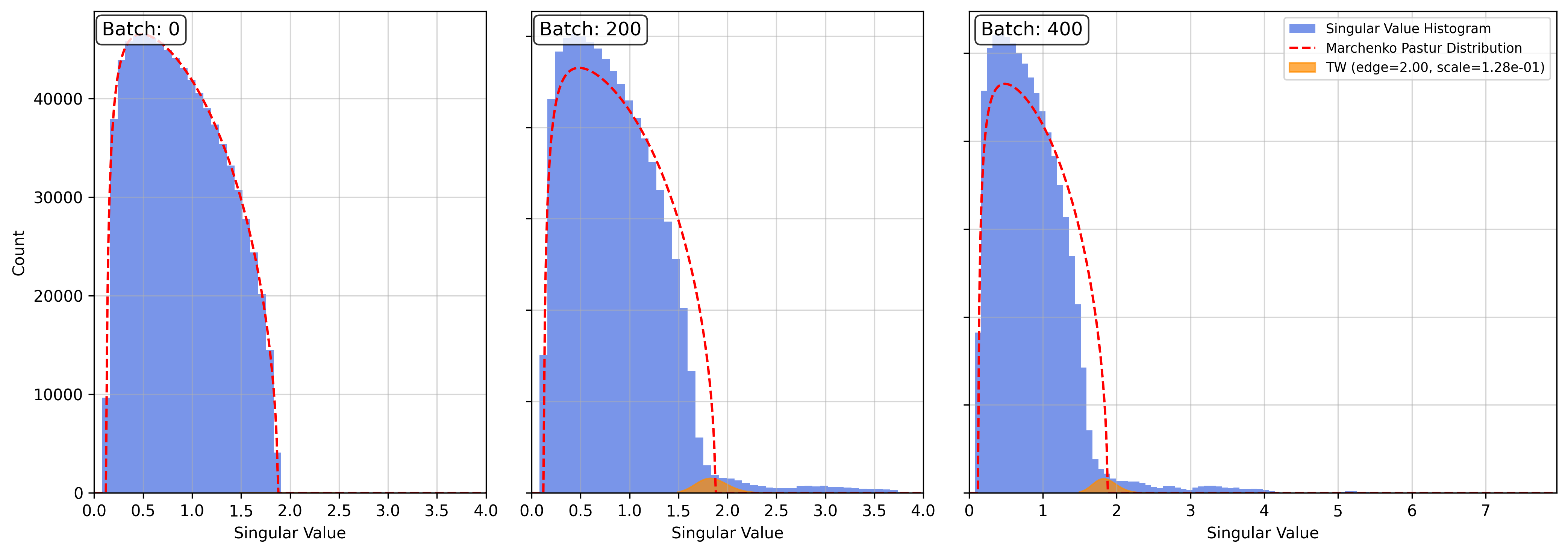}
    \caption{Singular‐value histograms at batches 0, 200, and 400, overlaid with the Marčenko–Pastur (MP) bulk law (red dashed) and the Tracy–Widom (TW) edge curve (green).}
    \label{fig:layer0}
\end{figure}
\vspace{-15pt}
\label{fig2}

In \hyperref[fig:layer0]{Figure 1}, at initialization the empirical spectrum adheres almost exactly to the MP prediction with no outliers beyond the TW edge, verifying random initialization assumptions. Our analysis begins by modeling the stochasticity of SGD with an isotropic noise term, resulting in a Langevin-type SDE. We acknowledge that this is an idealization; the true noise covariance of SGD is known to be anisotropic and parameter-dependent. However, this assumption provides a tractable starting point that allows us to establish a clear, analytical connection to the classical frameworks of Random Matrix Theory. This approach enables us to isolate and understand the fundamental repulsive dynamics that serve as a baseline for spectral evolution. We explicitly address the extension to the more realistic anisotropic case in our appendix (see \textbf{Proposition 6.16}), which we identify as a crucial direction for future work.
At batch 200, the MP/TW fits begin to underpredict mass near the spectrum's edge. This underprediction implies that growing correlations within the weight matrix diminish the effective size of the random matrix, consequently amplifying the dominance of edge statistics and fluctuations. Physically, this has the interpretation that we are still traversing either local minima or plateaus within the loss function landscape (producing weakly correlated learned features). At batch 400, this tail becomes increasingly pronounced: the bulk shifts rightward and a persistent shoulder of large singular values forms outside the MP support, showing that our matrix is becoming increasingly correlated. This is predicted to result from mostly gradient-derived information. Thus we understand the singular values beyond the TW threshold as highly correlated learned features. We empirically test our theory via \hyperref[algorithm_one]{Algorithm 1}. The algorithm uses gradient loss and singular value information to use the dynamic equation previously described to predict singular values over time.

\begin{figure}[H]
  \centering
  \vspace{-10pt}
  \begin{minipage}[c]{0.45\textwidth}
    \centering
    \includegraphics[width=\textwidth]{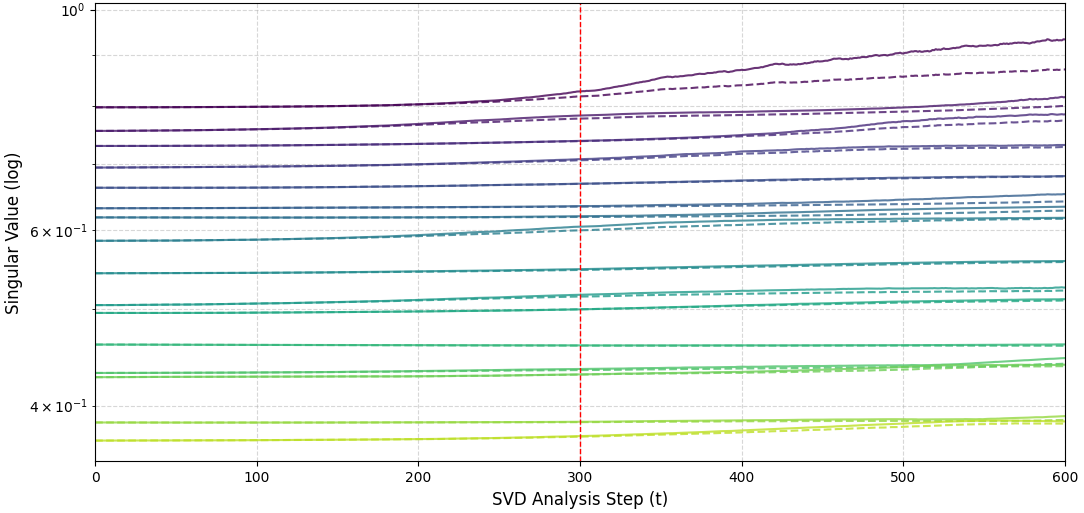}
    \captionof{figure}{Predicted singular values (dashed) versus true.}
    \label{fig:layer1}
  \end{minipage}%
  \hspace{0.05\textwidth}
  \begin{minipage}[c]{0.35\textwidth}
    \centering
    \includegraphics[width=\textwidth]{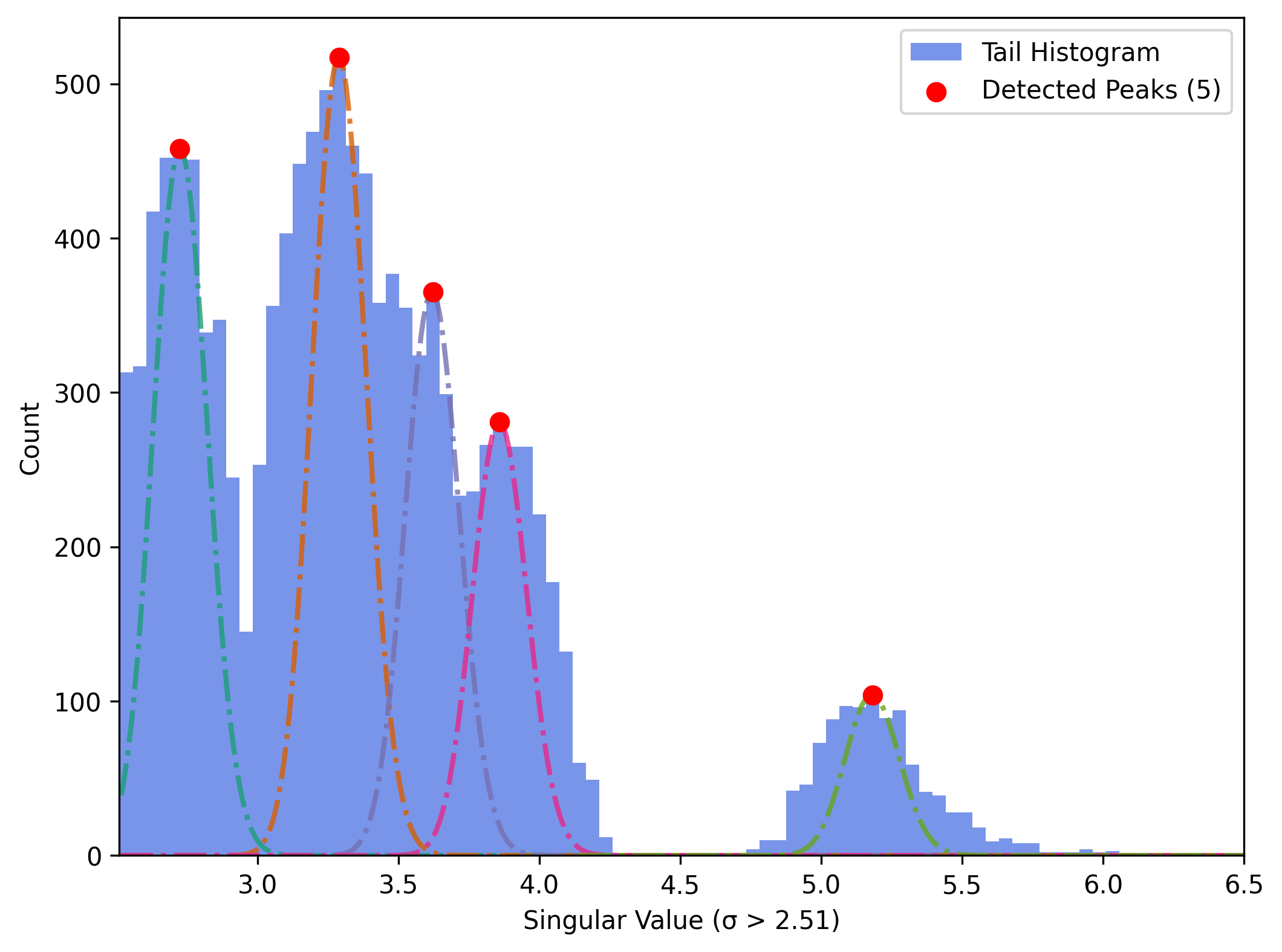}
    \captionof{figure}{Predicted heavy tails via \hyperref[thm3.2]{Theorem 3.2}.}
    \label{fig:edging3}
  \end{minipage}
  \vspace{-35pt}
\end{figure}

In \hyperref[fig:layer1]{Figure 2}, we track the top 8 singular values of a representative linear layer in a MLP over 800 training batches with CIFAR-10, plotting empirical trajectories against our bootstrap-drift predictions. The leading modes rise faster and the gradual increases of lower modes are reproduced by the prediction algorithm, with deviations starting around batch 300 as the spectrum begins to develop empirically observed `bulk+tail'  structure. This close alignment across all 8 modes up to the heavy‐tail regime demonstrates that our continuous‐time, matrix‐valued SDE framework accurately forecasts the full singular‐value dynamics from minimal gradient information. However, we hypothesize the anisotropic noise causes a bulk of the observed deviation for the singular values. This hypothesis is as follows:  the larger singular values might experience greater effects from anisotropic noise due to preferential alignment of noise with their dominant singular vectors and potentially larger Hessian components (see \textbf{\hyperref[616_prop]{Proposition 6.16}}). In \hyperref[fig:edging3]{Figure 3}, the fits predict the qualitative shape well, but underpredict counts significantly, which we attribute to the aforementioned anisotropic noise hypothesis.  

\begin{figure}[H]
    \centering
    \includegraphics[width=0.7\linewidth]{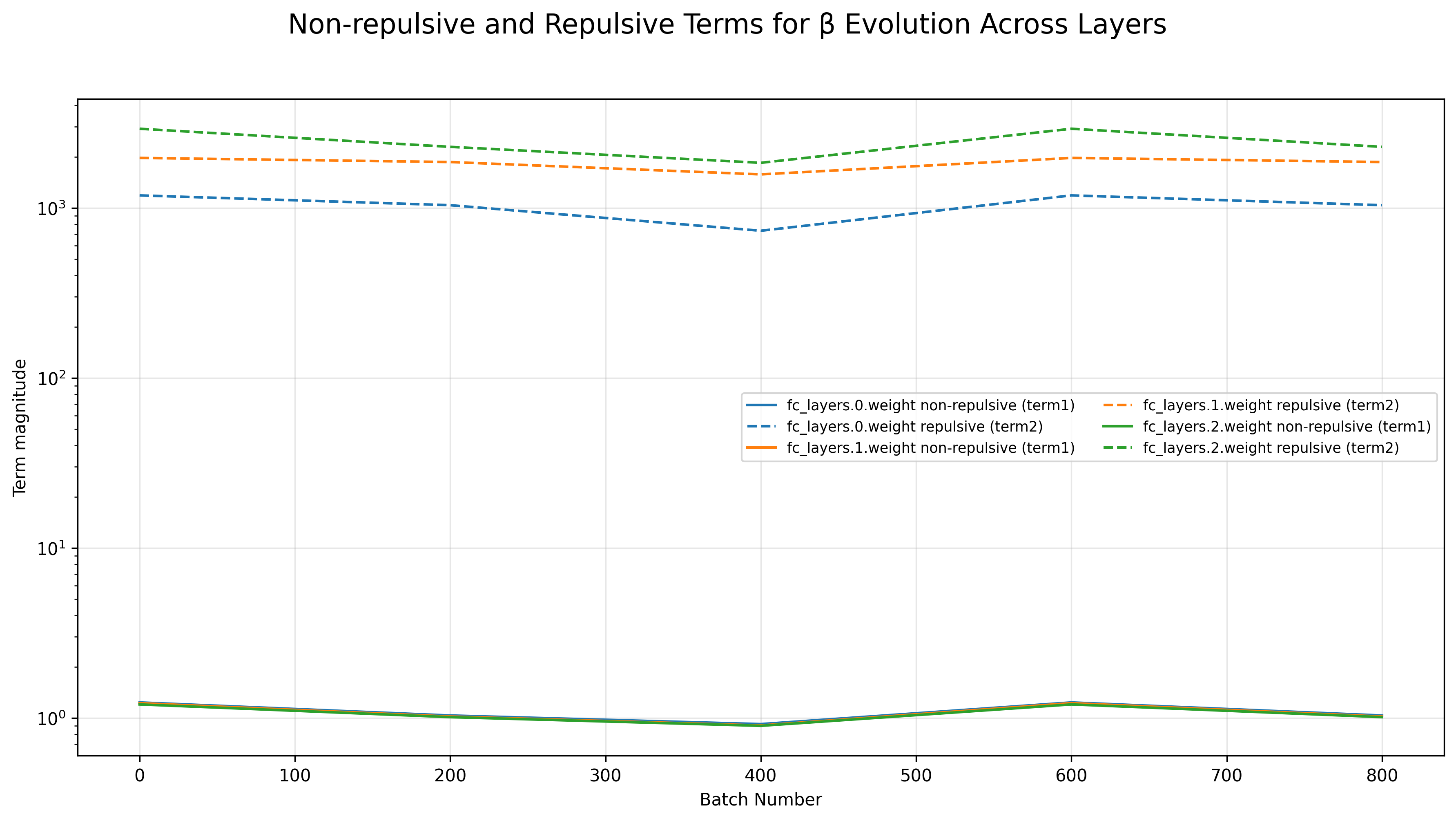}
    \captionof{figure}{Mean variance for gaussian distributions}
    \label{fig:layer4}
\end{figure}
\vspace{-32pt}
$$\beta_k = \frac{1}{\sqrt{\eta\lambda_k}} \left( \frac{\partial\lambda_k}{\partial t} + \sqrt{\eta}\lambda_k \frac{\partial L}{\partial W} - \sum_{j \neq k} \frac{\lambda_k}{\lambda_k - \lambda_j} \right)$$
\subsection{Noise Analysis}

As seen in the logarithmic plots above \hyperref[fig:layer4]{Figure 4} , the eigenvalue repulsive terms ($ \sum_{j \neq k} \frac{\lambda_k}{\lambda_k - \lambda_j} $)) are dominant in the noise contribution, and the repulsive force term is strongly correlated with the other two forces ($\frac{\partial\lambda_k}{\partial t} + \sqrt{\eta}\lambda_k \frac{\partial L}{\partial W} $). If the noise from mini-batch sampling were truly random or isotropic (i.e., of equal magnitude in all directions), there would be no statistical basis for such a correlation. The fluctuations would be independent of the system's internal dynamics. The existence of this correlation is therefore proof that the ``random'' kicks from SGD are systematically aligned with the geometry of the loss landscape and the current configuration of the eigenvalues.
 This suggests that the noise provides a sophisticated form of implicit regularization that goes beyond merely helping the model escape sharp minima. The noise actively reinforces the very dynamics that are beneficial for training. For instance, when two eigenvalues $\lambda_i$ and $\lambda_j$ get too close, the repulsive force between them increases. The observed correlation implies that the SGD noise will preferentially act in a direction that helps to separate them, effectively amplifying the natural repulsion. 

\subsection{Practical Applications}
\label{practical_applications}


By understanding how singular values evolve during training, we can design spectral-aware initialization schemes that pre-configure empirically observed `bulk+tail' structure bias to accelerate convergence, adaptive learning rate schedules that adjust $\eta/D$ ratios based on spectral concentration to balance exploration and convergence, and spectral-informed pruning that preserves these `bulk+tail' structure contributors while aggressively compressing bulk parameters. Our findings also may enable novel optimizers that incorporate spectral information to treat parameters in large singular value directions differently than bulk directions, and spectral-based monitoring that tracks the transition from MP-like to empirically observed `bulk+tail' structured spectra as an indicator of learning progress.

\section{Conclusion}
\label{conclusion}

We develop a continuous-time, matrix-valued SDE framework connecting SGD's microscopic dynamics to macroscopic spectral evolution, revealing that squared singular values follow Dyson Brownian motion and produce gamma-type distributions with power-law tails that explain the empirically observed `bulk+tail' structure in trained networks. Through controlled experiments, we demonstrate quantitative agreement between our predictions and observed spectral evolution, with our forecasting algorithm accurately predicting singular value trajectories until empirically observed `bulk+tail' structure emerges. While our current analysis assumes isotropic noise, future extensions to anisotropic SGD fluctuations could bridge the gap to real optimization dynamics and enable new preconditioning schemes.


\bibliography{main}
\bibliographystyle{icml2025}

\newpage
\section{Appendix}
\subsection{Main Theorem Proofs} \label{appendix:mainproofs}
\textbf{\hyperref[thm3.1]{Theorem 3.1}}
\begin{proof}
We regard the SGD update with isotropic noise as the Itô SDE on the weight matrix
\[
dW \;=\; A\,dt \;+\;\sqrt{2\eta D}\,d\mathcal W,
\qquad
A = -\,\eta\,\nabla_W\mathcal L,
\]
where $d\mathcal W$ is a matrix‐valued Wiener increment with independent entries.  Writing the SVD $W=U\Sigma V^T$ and denoting the $k$th singular value by $\sigma_k$, we apply Itô’s lemma to the scalar function $f(W)=\sigma_k(W)$.  First, by standard matrix‐perturbation theory,
\[
\nabla_W\sigma_k \;=\; u_k\,v_k^T,
\qquad
\Delta_W\sigma_k
=\frac{m-n+1}{2\sigma_k}
\;+\;\sum_{j\neq k}\frac{\sigma_k}{\sigma_k^2-\sigma_j^2}.
\]
Hence the general Itô formula
\[
df(W)
=\sum_{i,j}\frac{\partial f}{\partial W_{ij}}\,dW_{ij}
\;+\;\tfrac12\sum_{i,j,p,q}\frac{\partial^2 f}{\partial W_{ij}\partial W_{pq}}
\,dW_{ij}\,dW_{pq},
\]
together with
\[
dW_{ij} = A_{ij}\,dt + \sqrt{2\eta D}\,d\mathcal W_{ij},
\qquad
dW_{ij}\,dW_{pq} = 2\eta D\,\delta_{ip}\,\delta_{jq}\,dt,
\]
yields
\[
d\sigma_k
=\Bigl\langle\nabla_W\sigma_k,\;A\Bigr\rangle\,dt
\;+\;\eta D\,\Delta_W\sigma_k\,dt
\;+\;\sqrt{2\eta D}\,\langle\nabla_W\sigma_k,\;d\mathcal W\rangle.
\]
Substituting $\nabla_W\sigma_k = u_kv_k^T$ gives
\[
d\sigma_k
=\Bigl[-\eta\,u_k^T(\nabla_W\mathcal L)v_k
\;+\;\eta D\Bigl(\tfrac{m-n+1}{2\sigma_k}
+\sum_{j\neq k}\tfrac{\sigma_k}{\sigma_k^2-\sigma_j^2}\Bigr)\Bigr]dt
\;+\;\sqrt{2\eta D}\,d\beta_k(t),
\]
where $d\beta_k = u_k^T\,d\mathcal W\,v_k$ is a scalar Wiener increment.

Finally, set $\lambda_k=\sigma_k^2$ and apply Itô again:
\[
d\lambda_k
=2\sigma_k\,d\sigma_k + (d\sigma_k)^2
=2\sigma_k\,d\sigma_k + 2\eta D\,dt.
\]
Substituting the above expression for $d\sigma_k$ and simplifying yields
\[
d\lambda_k
=\Bigl[-2\sqrt{\lambda_k}\,\eta\,u_k^T(\nabla_W\mathcal L)v_k
\;+\;\eta D\,(m-n+3)
\;+\;2\eta D\sum_{j\neq k}\frac{\lambda_k}{\lambda_k-\lambda_j}\Bigr]dt
\;+\;2\sqrt{\lambda_k}\,\sqrt{2\eta D}\,d\beta_k(t),
\]
which is the desired result.
\end{proof}
 At initialization ($t=0$), the singular value spectrum of a random weight matrix is indeed described by the Marchenko-Pastur (MP) law. The contribution of our theorem is not to re-derive this initial state, but rather to characterize the initial dynamics that drive the spectrum away from this random configuration. The theorem formally describes the repulsive force ($\sum_{j \neq k}...$) induced by SGD's stochastic updates, which is the fundamental mechanism that introduces structure into the spectrum. We note that the dynamics of the squared singular values, $\lambda_k = \sigma_k^2$, are more precisely termed a Wishart process, a matrix-valued generalization related to Dyson Brownian motion.

\newpage
\textbf{\hyperref[thm3.2]{Theorem 3.2}}
\begin{proof}

Under the stationary mean‐field approximation with vanishing gradient, each squared singular value $\lambda_t$ evolves according to the one‐dimensional SDE
\[
d\lambda_t \;=\; (\alpha_0 - \beta_1\,\lambda_t)\,dt \;+\;\sqrt{8\eta D\,\lambda_t}\,dW_t,
\]
where $\alpha_0 = \eta D(m-n+3)$ and $\beta_1>0$ is a constant.  The corresponding stationary Fokker–Planck equation for the density $p(\lambda)$ is (by setting $\frac{\partial p(\lambda, t)}{\partial t} = 0$

\[ \frac{\partial p(\lambda, t)}{\partial t} = -\frac{\partial}{\partial \lambda} [(\alpha_0 - \beta_1\lambda) p(\lambda, t)] + \frac{1}{2} \frac{\partial^2}{\partial \lambda^2} [8\eta D \lambda p(\lambda, t)] \]

\[
0 \;=\; -\frac{d}{d\lambda}\bigl[(\alpha_0 - \beta_1\lambda)\,p(\lambda)\bigr]
+\frac12\,\frac{d^2}{d\lambda^2}\bigl[8\eta D\,\lambda\,p(\lambda)\bigr].
\]
Integrating once under the zero‐flux boundary condition gives
\[
(\alpha_0 - \beta_1\lambda)\,p(\lambda)
\;=\;
4\eta D\,\frac{d}{d\lambda}\bigl[\lambda\,p(\lambda)\bigr].
\]
Rearranging and separating variables, we have
\[
\frac{p'(\lambda)}{p(\lambda)}
=\Bigl(\frac{\alpha_0}{4\eta D}-1\Bigr)\frac1\lambda \;-\;\frac{\beta_1}{4\eta D}.
\]
Integrating gives us
\[
\ln p(\lambda)
=\Bigl(\frac{\alpha_0}{4\eta D}-1\Bigr)\ln\lambda
\;-\;\frac{\beta_1}{4\eta D}\,\lambda
\;+\;\ln C,
\]
so that
\[
p(\lambda)
=C\,\lambda^{\frac{\alpha_0}{4\eta D}-1}
\exp\!\Bigl(-\tfrac{\beta_1}{4\eta D}\,\lambda\Bigr),
\]
with $C$ fixed by normalization below
\[
C
=\frac{\bigl(\tfrac{\beta_1}{4\eta D}\bigr)^{\!\frac{\alpha_0}{4\eta D}}}
{\Gamma\!\bigl(\tfrac{\alpha_0}{4\eta D}\bigr)}.
\]
Noting $\alpha_0= \eta D(m-n+3)$ gives us the claimed form.

Finally, since $\sigma=\sqrt{\lambda}$, we find the push‐forward density
\[
p_\sigma(\sigma)
=2\sigma\,p(\sigma^2)
=2\frac{\bigl(\tfrac{\beta_1}{4\eta D}\bigr)^{\!\frac{m-n+3}{4}}}
{\Gamma\!\bigl(\tfrac{m-n+3}4\bigr)}
\,\sigma^{\frac{m-n+3}2-1}
\exp\!\Bigl(-\tfrac{\beta_1}{4\eta D}\,\sigma^2\Bigr),
\]
as desired.
\end{proof}
In deriving the stationary distribution in \textbf{Theorem 3.2}, we adopt a mean-field approximation. This approach decouples the interacting system of singular values, allowing us to analyze the dynamics of a single value within an effective potential. We recognize that this is a significant simplification, as the Coulomb-type repulsion term is fundamental to the transient dynamics. However, our goal here is to model the effective stationary state that emerges after prolonged training. In this limit, it is reasonable to approximate the complex, N-body interaction by an average restoring force, captured by the $\beta_1$ term. While this model neglects higher-order correlations, it yields a tractable Fokker-Planck equation whose solution successfully recovers the characteristic shape of the "bulk and tail" structure observed empirically.

\newpage
\subsection{Backpropagation as a Discrete Spatial-Temporal System}
In this section, we recast layer-wise backpropagation as a recursion in discrete space $x$ (layer index) and time $t$ (training iteration), laying the foundation for the continuous limit.
\begin{theorem}[Error Signal Recursion]
In the discrete spatial–temporal interpretation, the error signal $\delta(x,t)$ satisfies
\[
\delta(X_{\max},t)=\frac{\partial\mathcal L}{\partial a(X_{\max},t)}\odot f'\bigl(z(X_{\max},t)\bigr),
\quad
\delta(x,t)
=\bigl(W(x+1,t)^T\,\delta(x+1,t)\bigr)\odot f'\bigl(z(x,t)\bigr),
\]
for $x=X_{\max}-1,\dots,1$.
\end{theorem}
\begin{proof}
By definition $\delta(x,t)=\partial\mathcal L/\partial z(x,t)$.  At the boundary $x=X_{\max}$,
$$
\delta(X_{\max},t)
=\frac{\partial\mathcal L}{\partial a(X_{\max},t)}\cdot\frac{\partial a}{\partial z}
=\frac{\partial\mathcal L}{\partial a}\odot f'(z)\,.
$$
For $x<X_{\max}$ we apply the chain‐rule and get
\[
\delta(x,t)
=\frac{\partial\mathcal L}{\partial z(x,t)}
=\bigl(W(x+1,t)^T\,\partial \mathcal L/\partial z(x+1,t)\bigr)\odot f'(z(x,t))
=\bigl(W(x+1,t)^T\,\delta(x+1,t)\bigr)\odot f'(z(x,t)).
\]
\end{proof}

\begin{corollary}[Gradient Formulas]
The parameter gradients satisfy
\[
\frac{\partial\mathcal L}{\partial W(x,t)}
=\delta(x,t)\,a(x-1,t)^T,
\qquad
\frac{\partial\mathcal L}{\partial b(x,t)}
=\delta(x,t).
\]
\end{corollary}
\begin{proof}
We see that this immediately by $\partial z = a\,\partial W + \partial b$ and the definition of $\delta$.
\end{proof}
These theorems serve as the discrete foundation for Section~\ref{methodology}, enabling the PDE and SDE derivations.

\subsection{PDE Representation in Continuous Limit}
In this section, by letting the layer and time increments vanish, we derive PDEs describing the deterministic flow of parameters.

\begin{theorem}[Continuum PDE for Weight Evolution]
As $\Delta t,\Delta x\to0$, the discrete update
\(
W(x,t+\Delta t)-W(x,t)=-\eta\,\delta(x,t)\,a(x-1,t)^T
\)
converges formally to the PDE
\[
\partial_t W(x,t)
=-\eta\,\delta(x,t)\,f\bigl(z(x-1,t)\bigr)^T,
\quad
\delta(x,t)=\bigl(\partial_x^\dagger\delta\bigr)(x,t)\odot f'\bigl(z(x,t)\bigr),
\]
where $\partial_x^\dagger$ denotes the backward difference operator.
\end{theorem}
\begin{proof}
We first express
\[
\frac{W(x,t+\Delta t)-W(x,t)}{\Delta t}
=-\eta\,\delta(x,t)\,a(x-1,t)^T.
\]
Sending $\Delta t\to0$ yields the time‐derivative.  Meanwhile replacing the backward recursion for $\delta$ by the adjoint of the forward difference gives the continuous spatial dependence.
\end{proof}
This PDE describes the mean‐drift component of training and underlies our stochastic perturbations

\subsection{SGD as a Matrix‐Valued Itô SDE}
In this section, we show that random mini‐batch gradients introduce Brownian‐like noise into the weight dynamics.

\begin{theorem}[SGD as Itô SDE]
Under mini‐batch noise, with variance parameter $D$, the weight update
\(
W(x,t+\Delta t)-W(x,t)=-\eta\nabla_W\mathcal L(x,t)+\sqrt{2\eta D}\,\xi
\)
converges to the Itô SDE
\[
dW(x,t)
=-\eta\,\nabla_W\mathcal L\,dt
+\sqrt{2\eta D}\,d\mathcal W(x,t),
\]
where $d\mathcal W$ is matrix‐valued Brownian motion.
\end{theorem}
\begin{proof}
By central‐limit scaling of the mini‐batch noise we see that
\(\tfrac{1}{\sqrt{\Delta t}}\sum_i(\nabla\mathcal L_i-\nabla\mathcal L)\xrightarrow{d} \mathcal N(0,D)\),
hence in the limit $\Delta t\to0$ it becomes the Wiener increment $\sqrt{2\eta D}\,d\mathcal W$.
\end{proof}
This result justifies the isotropic noise term in the SDE.

\subsection{Itô’s Lemma for Singular Values}
In this section, we compute the drift and diffusion contributions to each singular value under the matrix Itô SDE.

\begin{lemma}[Gradient and Laplacian of $\sigma_k$]
Let $W=U\Sigma V^T$ be the SVD of $W\in\mathbb R^{m\times n}$ with $\Sigma_{kk}=\sigma_k>0$.  Then
\[
\nabla_W\sigma_k = u_k v_k^T,
\qquad
\Delta_W\sigma_k
=\frac{m-n+1}{2\sigma_k}
+\sum_{j\neq k}\frac{\sigma_k}{\sigma_k^2-\sigma_j^2}.
\]
\end{lemma}
\begin{proof}
The gradient is standard from matrix perturbation theory.  The Laplacian follows by differentiating twice and using orthonormality of singular vectors.
\end{proof}

\begin{theorem}[Itô SDE for $\sigma_k$]
Under
\(
dW=A\,dt+\sqrt{2\epsilon}\,d\mathcal W,
\)
the $k$th singular value obeys
\[
d\sigma_k
=\Bigl(\text{Tr}((u_kv_k^T)^TA)+\epsilon\,\Delta_W\sigma_k\Bigr)\,dt
+\sqrt{2\epsilon}\,d\beta_k,
\]
where $d\beta_k=u_k^T\,d\mathcal W\,v_k$ is scalar Brownian motion.
\end{theorem}
\begin{proof}
We apply the general Itô formula
\[
df(W)
=\sum_{i,j}f_{ij}\,dW_{ij}
+\tfrac12\sum_{i,j,k,l}f_{ij,kl}\,dW_{ij}dW_{kl},
\]
with $f(W)=\sigma_k(W)$, and use
\(
dW_{ij}dW_{kl}=2\epsilon\,\delta_{ik}\delta_{jl}\,dt,
\)
together with the lemma above.
\end{proof}
These theorems form the basis for the interacting SDEs of singular values.

\subsection{Mapping to Dyson Brownian Motion}\label{appendix:dbm}
In this section, we show that in the zero‐gradient regime, squared singular values follow a Dyson‐type interacting particle SDE.

\begin{theorem}[Dyson–SDE Identification]
Let $\lambda_k=\sigma_k^2$.  Then in the gradient‐flat regime $\nabla\mathcal L\approx0$,
\[
d\lambda_k
=\Bigl(\eta D(m-n+3)+2\eta D\sum_{j\neq k}\frac{\lambda_k}{\lambda_k-\lambda_j}\Bigr)\,dt
+2\sqrt{2\eta D\,\lambda_k}\,d\beta_k,
\]
which after time‐rescaling becomes the $\beta=1$ Dyson Brownian motion
\(
dY_k
=\Bigl(\tfrac{m-n+3}2+\sum_{j\neq k}\tfrac{Y_k}{Y_k-Y_j}\Bigr)\,ds
+2\sqrt{Y_k}\,dW_k.
\)
\end{theorem}
\begin{proof}
Compute $d\lambda_k$ via Itô on $f(\sigma)=\sigma^2$, use the previous SDE, drop the gradient term, and choose $s=t/(2\eta D)$ so that the prefactors match exactly the canonical form.
\end{proof}
This theorem helps to explain the eigenvalue repulsion and spectral spreading.

\subsection{Stationary Fokker–Planck and Gamma Law}
We show that solving the steady‐state Fokker–Planck PDE for one SDE yields a Gamma‐family density.

\begin{proposition}[Stationary Density is Gamma‐type]
For the one‐particle SDE
\(
d\lambda_t=(\alpha_0-\beta_1\lambda_t)\,dt+\sqrt{8\eta D\,\lambda_t}\,dW_t,
\)
the stationary solution of the Fokker–Planck equation is
\[
p(\lambda)
\propto \lambda^{\frac{\alpha_0}{4\eta D}-1}
\exp\!\Bigl(-\tfrac{\beta_1}{4\eta D}\,\lambda\Bigr),
\quad \lambda>0.
\]
\end{proposition}
\begin{proof}
Setting the time‐derivative to zero, we get
\[
0=-\partial_\lambda\bigl[(\alpha_0-\beta_1\lambda)p\bigr]
+\tfrac12\partial^2_\lambda\bigl[8\eta D\,\lambda p\bigr].
\]
Integrating once under zero‐flux boundary conditions and separating variables, we get
\(\tfrac{p'}p=(\tfrac{\alpha_0}{4\eta D}-1)\tfrac1\lambda-\tfrac{\beta_1}{4\eta D}\),
and exponentiate to obtain the Gamma‐form.
\end{proof}
Obtaining a Gamma form, we see that this justifies the heavy‐tail exponents observed in our experimental results.

\subsection{Further Mathematical Analysis}
In this section, we assemble classical random–matrix and integral–transform results that underpin our spectral SDE framework.  First, we recall the Marchenko–Pastur and Tracy–Widom edge laws, which describe the untrained and boundary fluctuations of large random weight matrices.  Then we turn to Hilbert transforms and stationary mean-field equations, which provide the macroscopic density needed in our Fokker–Planck analysis of singular-value dynamics.
\subsubsection{Derivation of the Stochastic Term $\beta$}
\begin{proposition}[Solving for the Stochastic Term]\label{prop:beta_solve}
Given the stochastic differential equation for the temporal evolution of an eigenvalue $\lambda_k$, the stochastic term $\beta$ can be isolated.
\end{proposition}
\begin{proof}
We begin with the SDE describing the evolution of the eigenvalue $\lambda_k$, which includes terms for the gradient of the loss, eigenvalue repulsion, and a stochastic component driven by $\beta$:
\begin{equation}
    \frac{\partial\lambda_k}{\partial t} = -\sqrt{\eta}\lambda_k \frac{\partial L}{\partial W} + \sum_{j \neq k} \frac{\lambda_k}{\lambda_k - \lambda_j} + \sqrt{\eta\lambda_k} \beta
    \label{eq:start_sde}
\end{equation}

\begin{equation}
    \frac{\partial\lambda_k}{\partial t} + \sqrt{\eta}\lambda_k \frac{\partial L}{\partial W} - \sum_{j \neq k} \frac{\lambda_k}{\lambda_k - \lambda_j} = \sqrt{\eta\lambda_k} \beta
    \label{eq:rearranged}
\end{equation}
Finally, we divide by the coefficient of $\beta$, which is $\sqrt{\eta\lambda_k}$, to obtain the expression for the stochastic term:
\begin{equation}
    \beta = \frac{1}{\sqrt{\eta\lambda_k}} \left( \frac{\partial\lambda_k}{\partial t} + \sqrt{\eta}\lambda_k \frac{\partial L}{\partial W} - \sum_{j \neq k} \frac{\lambda_k}{\lambda_k - \lambda_j} \right)
    \label{eq:final_beta}
\end{equation}
This completes the derivation.
\end{proof}
\subsubsection{Scaling Limits and Tracy-Widom Limit}
At initialization, our weight matrices follow Wishart (or MP) statistics.  Understanding the bulk and edge of this spectrum is essential both to verify our isotropic SGD noise reproduces classical limits and to identify regimes where empirically observed `bulk+tail' structured deviations occur.
\begin{lemma}{Bulk and Edge of the Marchenko-Pastur Law}\label{mplaw}
We first let $M = \frac{1}{n}XX^T$ be an $m\times m$ Wishart matrix with $X\in\mathbb R^{m\times n}$ having i.i.d.\ entries of variance $1$.  As $m,n\to\infty$ with $m/n\to \gamma\in(0,1]$, the empirical spectral distribution of $M$ converges to the Marchenko–Pastur density
\[
\rho_{MP}(x)
=\frac{\sqrt{(\lambda_+ - x)(x-\lambda_-)}}{2\pi\gamma x},
\quad x\in[\lambda_-,\lambda_+],
\]
where
\[
\lambda_{\pm}=(1\pm\sqrt\gamma)^2.
\]
\end{lemma}
This proof is presented in \citep{MarchenkoPastur1967}, where the Stieltjes transform of $M$ is utilized. In all, this bulk law justifies our use of MP fits at \(t=0\), and sets the stage for tracking departures under SGD noise.

\begin{lemma}[Edge Scaling Constants]\label{edgescaling}
Under the same regime, we let $\lambda_{(1)}$ be the largest eigenvalue of $M$. We define
\[
\mu_{m,n} = \lambda_+,\qquad
\sigma_{m,n} = (\lambda_+)^{1/2}\,\frac{(1+\gamma^{-1/2})^{1/3}}{n^{2/3}}.
\]
Then the centered and scaled variable
\[
\chi_{m,n}
= \frac{\lambda_{(1)} - \mu_{m,n}}{\sigma_{m,n}}
\]
has fluctuations on order one as $m,n\to\infty$.
\end{lemma}

\begin{corollary}[Tracy–Widom $F_1$ Limit]\label{twlimit}
It was shown in \citep{tracyOrthogonalSymplecticMatrix1996} and this was proven by expressing the gap probability as a Fredholm determinant of the Airy kernel, thereby yielding the limit.
\[
\lim_{m,n\to\infty}
\mathbb P\bigl(\chi_{m,n}\le s\bigr)
= F_1(s),
\]
where $F_1$ is the Tracy–Widom distribution for $\beta=1$ (real symmetric ensembles).
\end{corollary}

\begin{definition}[Airy Kernel and Process]
The \emph{Airy kernel} is
\[
K_{\mathrm{Ai}}(x,y)
=\frac{\mathrm{Ai}(x)\,\mathrm{Ai}'(y)-\mathrm{Ai}'(x)\,\mathrm{Ai}(y)}{x-y},
\]
and the \emph{Airy process} $\{\mathcal A(t)\}$ is the determinantal process with kernel
\[
K_{\mathrm{Ai}}(t_1,\xi_1;\,t_2,\xi_2)
= \int_0^\infty e^{-u(t_2-t_1)}\mathrm{Ai}(\xi_1+u)\,\mathrm{Ai}(\xi_2+u)\,du.
\]
\end{definition}

\begin{remark}
The largest eigenvalue fluctuations of Dyson’s Brownian motion (with $\beta=1$) also converge to the Airy process, giving a dynamical Tracy–Widom law for $\lambda_{(1)}(t)$ under appropriate time scaling.
\end{remark}

\subsubsection{Hilbert Transform and Stationary Density}
To derive the macroscopic spectral density under our isotropic SDE, we solve a stationary Fokker–Planck equation via Hilbert transforms.  The lemma below gives a closed-form for power-law inputs, enabling the Gamma-like stationary density.
\begin{lemma}[Hilbert Transform of Power Law Densities]
If 
\(\rho(x)=C\,x^\alpha\) on \([0,R]\), then its finite‐interval Hilbert transform is
\[
H[\rho](\lambda)
=\frac{1}{\pi}\text{PV}\left(\!\int_0^R\frac{C\,x^\alpha}{x-\lambda}\,dx\right)
= C\,\lambda^\alpha\cot(\pi\alpha)+O(1),
\]
for $\lambda\in(0,R)$ and $\alpha\notin\mathbb Z$.
\end{lemma}

\begin{corollary}[Stationary Density at Large $r$]\label{rlimit}
Under the quadratic mean‐field potential and isotropic noise, the large‐$r$ stationary $\rho_{st}(\lambda)$ solving
\[
\eta D(m-n+3)-2\pi\eta D\,\frac{d}{d\lambda}\bigl(\lambda\,H[\rho_{st}](\lambda)\bigr)=0
\]
behaves to leading order like
\[
\rho_{st}(\lambda)\propto \lambda^{\frac{1}{4}(m-n+3)-1},
\]
recovering the Gamma‐type density in the effective single‐particle Fokker–Planck.
\end{corollary}
This provides the explicit stationary spectrum that emerges from our isotropic SDE.
\subsubsection{Anisotropic Analysis}
Thus far, our analysis has assumed that the random fluctuations in SGD are isotropic meaning every direction in parameter space experiences the same noise strength. In practice, however, noise can be highly direction‐dependent—layers, singular modes, or even individual parameters often see very different variance due to batch structure, learning‐rate schedules, or architecture specifics. Accounting for this anisotropy is crucial if we hope to predict not only the locations of singular values, but also the relative spreading and alignment of singular vectors over training. The following proposition shows how the general Itô‐lemma approach naturally incorporates a full covariance structure $\Sigma(W,t)$, yielding additional second–derivative corrections to the drift and a directionally weighted diffusion term. This richer SDE then serves as the foundation for a non‐homogeneous Dyson‐type PDE in the mean‐field limit, capable of capturing empirically observed anisotropic spectral evolution. We carry out the derivations below, and we leave experimentation for anisotropic analysis for future work.

We have for anisotropic analysis,
\[
dW = -\,\eta\,\nabla_W \mathcal L\,dt \;+\; B(W,t)\,d\mathcal B_t,
\]
where \(B(W,t)B(W,t)^T = 2\eta\,\Sigma(W,t)\) and \(d\mathcal B_t\) is our standard matrix Wiener process.
\begin{proposition}[Anisotropic Noise - Changes to SDE for Singular Values]\label{616_prop}
We let \(W(t)=U(t)\,\Sigma(t)\,V(t)^T\) be the SVD of \(W\), and we denote \(\sigma_k(t)\) the \(k\)$^{\text{th}}$ singular value. Then under the dynamics above, we have
\[
d\sigma_k = 
\Bigl\langle \nabla_W\sigma_k \,,\, -\eta\,\nabla_W\mathcal L\Bigr\rangle\,dt
\;+\;\frac12\,\sum_{i,j,p,q}
\frac{\partial^2\sigma_k}{\partial W_{ij}\partial W_{pq}}
\bigl[2\eta\,\Sigma(W,t)\bigr]_{ip,jq}\,dt
\;+\;\bigl\langle\nabla_W\sigma_k\,,\,B(W,t)\,d\mathcal B_t\bigr\rangle.
\]
Since \(\nabla_W\sigma_k = u_kv_k^T\) and 
\(\displaystyle \frac{\partial^2\sigma_k}{\partial W_{ij}\partial W_{pq}}\)
is known from matrix‐perturbation theory, the drift becomes
\[
u_k^T\!\bigl(-\eta\nabla_W\mathcal L\bigr)v_k
\;+\;\eta\;\mathrm{Tr}\!\Bigl[\Sigma(W,t)\,\nabla^2_W\sigma_k\Bigr],
\]
and the diffusion term is 
\(\sqrt{2\eta}\;\langle u_kv_k^T,\;\sqrt{\Sigma(W,t)}\,d\mathcal B_t\rangle.\)
\end{proposition}
\begin{proof}
For a scalar \(f(W)\), we know by Ito's Lemma that
\[
df \;=\; \sum_{i,j}\frac{\partial f}{\partial W_{ij}}\,dW_{ij}
\;+\;\frac12\sum_{i,j,p,q}\frac{\partial^2 f}{\partial W_{ij}\,\partial W_{pq}}
\;dW_{ij}\,dW_{pq}.
\]

Now, we proceed to substitute \(dW_{ij}\). We get that 
\[
dW_{ij}
= -\eta\,(\nabla_W \mathcal L)_{ij}\,dt
\;+\;\sum_{\alpha,\beta}B_{ij,\alpha\beta}\,d\mathcal B_{\alpha\beta}.
\]
Hence
\[
\sum_{i,j}\frac{\partial f}{\partial W_{ij}}\,dW_{ij}
=
\underbrace{\sum_{i,j}\frac{\partial f}{\partial W_{ij}}\bigl[-\eta(\nabla_W\mathcal L)_{ij}\bigr]}_{=\langle\nabla_W f,\,-\eta\,\nabla_W\mathcal L\rangle}\,dt
\;+\;
\sum_{i,j,\alpha,\beta}
\frac{\partial f}{\partial W_{ij}}
\,B_{ij,\alpha\beta}\,d\mathcal B_{\alpha\beta}.
\]
We see that only the noise part contributes second‐order terms, hence we have
\[
dW_{ij}\,dW_{pq}
=
\Bigl(\sum_{\alpha,\beta}B_{ij,\alpha\beta}\,d\mathcal B_{\alpha\beta}\Bigr)
\Bigl(\sum_{\gamma,\delta}B_{pq,\gamma\delta}\,d\mathcal B_{\gamma\delta}\Bigr)
=
\sum_{\alpha,\beta}B_{ij,\alpha\beta}\,B_{pq,\alpha\beta}\;dt
=
2\eta\,\Sigma_{ip,jq}(W,t)\,dt.
\]
Thus, we have
\[
\frac12\sum_{i,j,p,q}
\frac{\partial^2 f}{\partial W_{ij}\,\partial W_{pq}}
\,dW_{ij}\,dW_{pq}
=
\underbrace{\eta
\sum_{i,j,p,q}
\frac{\partial^2 f}{\partial W_{ij}\,\partial W_{pq}}
\,
\Sigma_{ip,jq}(W,t)}_{=\frac12\sum(\partial^2 f)\,[2\eta\,\Sigma]_{ip,jq}}
\;dt.
\]
Now, we proceed to group all the \(dt\) terms, giving us the drift term below
\[
\langle\nabla_W f,\,-\eta\,\nabla_W\mathcal L\rangle
\;+\;\frac12\sum_{i,j,p,q}\frac{\partial^2 f}{\partial W_{ij}\partial W_{pq}}
\bigl[2\eta\,\Sigma(W,t)\bigr]_{ip,jq},
\]
and the remaining stochastic term is the martingale term given by
\[
\sum_{i,j,\alpha,\beta}
\frac{\partial f}{\partial W_{ij}}\,
B_{ij,\alpha\beta}\,d\mathcal B_{\alpha\beta}
=
\bigl\langle\nabla_W f,\;B(W,t)\,d\mathcal B_t\bigr\rangle.
\]

Finally, from matrix perturbation theory, we know that
\[
\nabla_W\sigma_k = u_k\,v_k^T,
\quad
\frac{\partial^2\sigma_k}{\partial W_{ij}\,\partial W_{pq}}
\;=\;\bigl[\nabla^2_W\sigma_k\bigr]_{ij,pq},
\]
so the final SDE is
\[
d\sigma_k
=
\underbrace{\langle u_kv_k^T,\,-\eta\,\nabla_W\mathcal L\rangle}_{\text{drift from loss}}
\,dt
\;+\;
\frac12
\sum_{i,j,p,q}
\bigl[\nabla^2_W\sigma_k\bigr]_{ij,pq}
\bigl[2\eta\,\Sigma(W,t)\bigr]_{ip,jq}
\,dt
\;+\;
\bigl\langle u_kv_k^T,\;B(W,t)\,d\mathcal B_t\bigr\rangle,
\]
as proposed.
\end{proof}
\begin{lemma}
\label{lma1}
(\textit{Estimating the Diffusion Constant for Stationary Distribution Fitting}) \\
In order to connect our theoretical diffusion coefficient \(D\) to observable quantities during training, we employ a simple dimensional‐analysis argument. The diffusion term in our singular‐value SDE has units of \(\text{(singular‐value)}^2\) per unit time, so \(D\) must scale like the variance of singular‐value changes divided by the time step. Empirically, the broadening of the spectrum is characterized by the gap between the largest mode and a representative central mode—here taken as \(\sigma_{\max}-\sigma_{\mathrm{med}}\). Over \(t_b\) batch updates, this gap typically increases by an amount on the order of its own magnitude. Matching units then gives
\begin{align*}
    [D] = \frac{L^2}{s} \mapsto \frac{(\sigma_{max} - \sigma_{med})^2}{t_b}
\end{align*}
where $\sigma_{max} $ corresponds the maximum singular value, $ \sigma_{med}$ corresponds to the median singular value, and $t_b$ is the time (which is represented as the batch update number in our spatiotemporal interpretation).
\end{lemma}
\begin{lemma}
\label{lma2}
(\textit{Estimating the Noise Constant $\beta_1$ for Stationary Distribution Fitting}) \\
Letting $L(w) = \frac{1}{N} \sum_{i=1}^{N} L_i(W)$ be the loss function, and defining the \textbf{batch gradient} (true gradient) as:
$$\nabla L(w) = \frac{1}{N} \sum_{i=1}^{N} \nabla L_i(W)$$
and \textbf{minibatch gradient} for a randomly sampled minibatch $S_t$ of size $B$ as:
$$\nabla L_{S_t}(w) = \frac{1}{B} \sum_{j \in S_t} \nabla L_j(w)$$
we model the \textbf{SGD noise} for a minibatch $S_t$ as the difference:
$$\beta_1(W, S_t) = \nabla L_{S_t}(W) - \nabla L(W)$$
Implying $$\| \beta_1(W, S_t) \|^2 = \| \nabla L_{S_t}(W) - \nabla L(w) \|^2$$

The minibatch gradient is an unbiased estimator: $\mathbb{E}_{S_t}[\nabla L_{S_t}(w)] = \nabla L(w)$.
The \textbf{variance of the minibatch gradient}, which is the formal measure of SGD noise, is given by the expected squared norm of the noise term:
$$ \text{Var}(\nabla L_{S_t}(w)) \equiv \mathbb{E}_{S_t} \left[ \| \nabla L_{S_t}(w) - \nabla L(w) \|^2 \right] \equiv \mathbb{E}_{S_t} [ \| \beta_1(W, S_t)  \|^2$$
Thus we use this empirically determined value of $\beta_1$ for our fits for the stationary distributions. 

\end{lemma}
\section{Algorithm Details}
\label{algorithm_one}
\begin{algorithm}[H]
\caption{Predicting Singular–Value Dynamics via Bootstrapped Drift}\label{alg:svd-prediction-compact}
\begin{algorithmic}[1]
  \State \textbf{Input:} $W^{(0)}\!\in\!\mathbb R^{m\times n},\,\eta,\,T,\,k$
  \State \textbf{Output:} $\{\sigma^{(t)}\}_{t=0}^T$
  \Statex
  \State $[U,\Sigma,V]\gets\mathrm{svd}(W^{(0)})$;\quad
         $U_k\gets U_{: ,1\!:\!k},\;\sigma\gets\diag(\Sigma)_{1\!:\!k},\;V_k\gets V_{: ,1\!:\!k}$
  \For{$t=0,\dots,T-1$}
    \State $G\gets -\eta\,\nabla_W\ell^{(t)}$
    \State $M\gets U_k^\top G\,V_k$
    \For{$i=1,\dots,k$}
      \State $\Delta\sigma_i\gets M_{ii}$
      \State $\displaystyle du_i,\;dv_i\;=\;\sum_{j\neq i}
        \Bigl(\tfrac{M_{ji}}{\sigma_i-\sigma_j+\varepsilon}+\tfrac{M_{ij}}{\sigma_i+\sigma_j+\varepsilon}\Bigr)\,(U_k[:,j],\,V_k[:,j])$
      \State $\sigma_i\gets\max(\sigma_i+\Delta\sigma_i,\,\delta)$
      \State $\widetilde U_i\gets U_k[:,i]+du_i,\quad \widetilde V_i\gets V_k[:,i]+dv_i$
    \EndFor
    \State $U_k,V_k\gets\mathrm{orth}\bigl([\widetilde U_1,\dots,\widetilde U_k]\bigr),
                         \;\mathrm{orth}\bigl([\widetilde V_1,\dots,\widetilde V_k]\bigr)$
    \State Align signs of $U_k,V_k$ columns with previous
  \EndFor
  \State \Return $\{\sigma^{(t)}\}$
\end{algorithmic}
\end{algorithm}

\subsection{Computational Complexity Analysis}

Algorithm~\ref{alg:svd-prediction-compact} offers significant computational advantages over naive approaches that recompute the full SVD at each time step. We analyze the complexity for an $m \times n$ weight matrix over $T$ time steps, tracking the top $k$ singular values.

\paragraph{Initial Setup.} The initial SVD computation (Line 3) requires $O(\min(m^2n, mn^2))$ operations, which is performed only once.

\paragraph{Per-Timestep Complexity.} For each of the $T$ time steps, the algorithm performs:
\begin{itemize}
    \item \textbf{Gradient computation}: $O(G)$ operations, where $G$ depends on the specific loss function and network architecture.
    \item \textbf{Projection}: Computing $M = U_k^T G V_k$ requires $O(kmn)$ operations.
    \item \textbf{Singular value updates}: For each of the $k$ singular values, the drift computation involves $O(k)$ operations and the singular vector updates require $O(k \max(m,n))$ operations, yielding $O(k^2 \max(m,n))$ total.
    \item \textbf{Orthogonalization}: The Gram-Schmidt orthogonalization step costs $O(k^2 \max(m,n))$ operations.
\end{itemize}

\paragraph{Total Complexity.} The overall computational complexity is:
$$O\bigl(\min(m^2n, mn^2) + T(G + kmn + k^2 \max(m,n))\bigr)$$

\paragraph{Efficiency Gains.} When $k \ll \min(m,n)$ (typically $k \leq 10$ for the leading modes), our algorithm achieves substantial speedups compared to naive full SVD recomputation at each step, which would require $O(T \min(m^2n, mn^2))$ operations. For large matrices where $m, n \gg k$, the per-timestep cost reduces from $O(\min(m^2n, mn^2))$ to $O(kmn + k^2 \max(m,n))$, representing a factor of $\Theta(\min(m,n)/k)$ improvement in the SVD-related computations.

\section{Additional Experimental Details} \label{appendix:models}

\paragraph{SGD.} Our use of SGD follows the classic Ornstein–Uhlenbeck approximation for constant‐rate noise \citep{Mandt2017}, while observed anisotropies in batch‐size and learning‐rate interactions \citep{Jastrzebski2017ThreeFactors} directly inform our extension to non‐isotropic noise. Choosing a quadratic mean‐field potential for large‐width spectral dynamics is supported by recent convergence results in overparameterized models \citep{ChizatBach2018, Rotskoff2018}.

\paragraph{GPT2.} We use the nanoGPT implementation \citep{karpathyKarpathyNanoGPT2025} which follows the transformer decoder-only architecture with four transformer layers, four attention heads per layer, and 256-dimensional embeddings. The learning rate starts at $5 * 10^{-4}$ with cosine decay to $5 * 10^{-5}$. We use a batch size of 12 sequences of 256 tokens each. 

\paragraph{Vision Transformer (ViT).} ViT is configured with two encoder layers, four attention heads, and a 256-dimensional embedding. Inputs ($H \times H$) are segmented into patches ($P \times P$), transformed by standard Transformer blocks (FFN expansion ratio $\alpha=2$), and classified via a linear head initialized as $w \sim \mathcal{N}(0, 1/\sqrt{H_{dim}})$. We set $(H,P)=(28,7)$ for MNIST and $(32,8)$ for CIFAR-100.

\paragraph{Multilayer Perceptron (MLP).} Our MLP comprises three hidden layer of 1024 dimensions. Weight matrices are initialized from $\mathcal{N}(0, 1/\mathrm{fan_{in}})$, with biases initialized to zero.

\textbf{Other Considered Models.} ResNet architecture \citep{heDeepResidualLearning2015} is not used as it consists mostly of convolutional layers with structured weight sharing patterns making some spectral properties less interpretable for understanding loss landscapes.

\section{Additional Experimental Results}
\begin{figure}[H]
    \centering
    \includegraphics[width=1.0\linewidth]{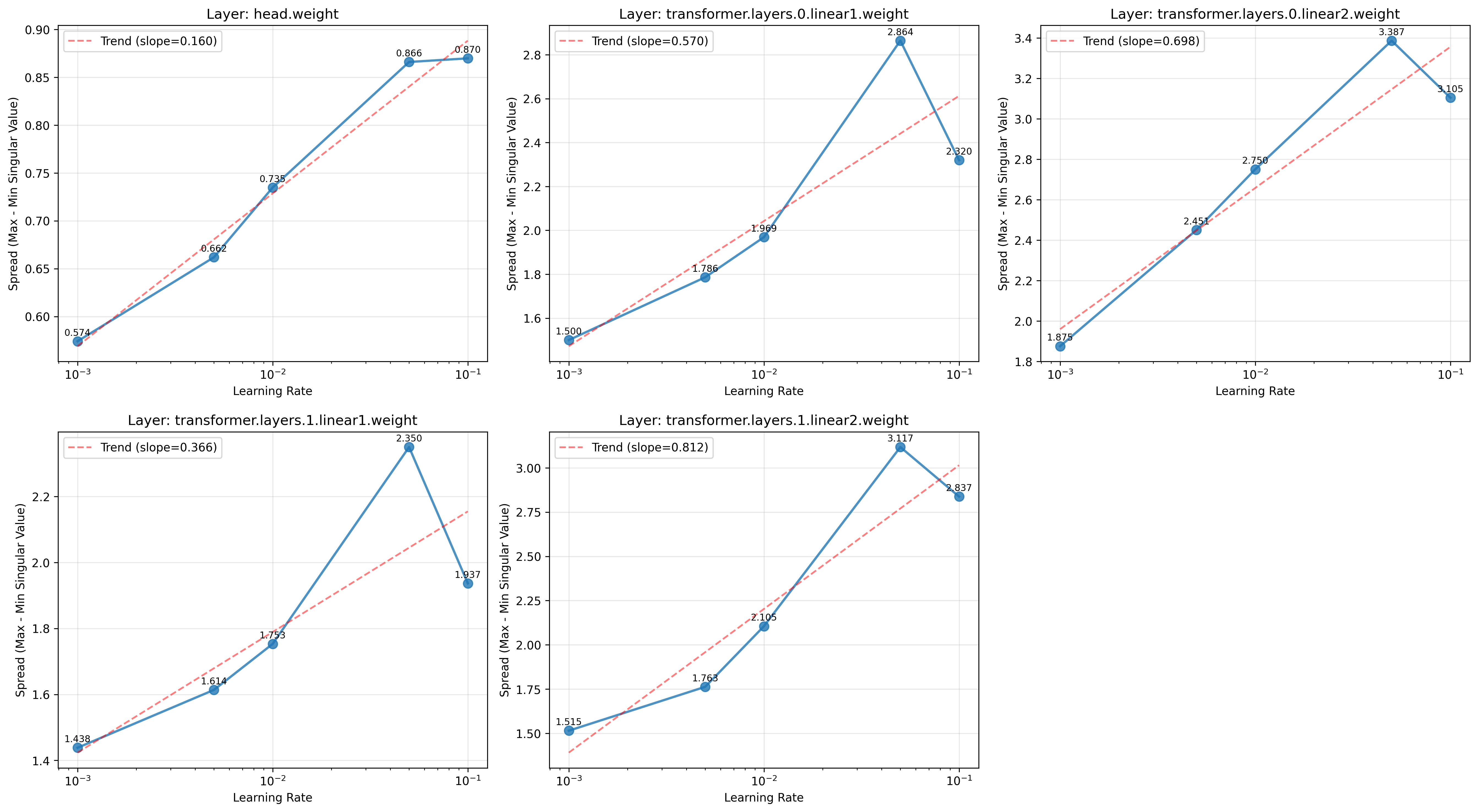}
    \caption{Spread of singular values (max–median) versus learning rate for different vision transformer weight matrices, with red dashed least-squares trends and slopes indicating sensitivity.}
    \label{fig:layer_vit}
\end{figure}
\label{fig5}
Across all layers, increasing the learning rate from $10^{-3}$ to $10^{-1}$ amplifies the spectral spread, indicating that higher noise levels drive greater anisotropy in the weight matrix. Moreover, the fitted trend‐line slopes reveal that the second feed-forward projection in each layer is most sensitive to learning‐rate scaling. In particular, in layer 1 (slope $\approx$ 0.81)—whereas the output head’s weights remain comparatively stable (slope $\approx$ 0.16). These results show that isotropic SGD induces layer-dependent spectral broadening, with deeper feed-forward blocks experiencing the strongest effect. When a certain critical learning rate is hit, we see that the spread decreases, indicating more uniformity in singular values, potentially implying that fewer features are being learnt by the model, thus meriting further investigation to understand this phenomenon.

\end{document}